\documentclass[10pt]{article}
\usepackage{amsmath,amsthm,amsfonts,amssymb,latexsym,mathtools,graphicx,stmaryrd}
\usepackage{empheq,fancybox}
\usepackage[noadjust]{cite}
\usepackage{algorithm}
\usepackage{algpseudocode}
\usepackage[pdfpagemode=UseNone,pdfstartview=FitH]{hyperref}

\newcommand{\OCMXXIV}{Vovk:2021}

\allowdisplaybreaks[4]

\newtheorem{theorem}{Theorem}
\newtheorem{lemma}[theorem]{Lemma}
\newtheorem{corollary}[theorem]{Corollary}
\newtheorem{proposition}[theorem]{Proposition}
\theoremstyle{definition}

\newtheorem{remark}{Remark}

\DeclareMathOperator{\Prob}{\mathbb{P}}
\DeclareMathOperator{\Expect}{\mathbb{E}}

\newcommand*{\dd}{\,\mathrm{d}}

\newcommand{\R}{\mathbb{R}}

\newcommand{\FFF}{\mathcal{F}}
\newcommand{\GGG}{\mathcal{G}}

\title{Testing for concept shift online}
\author{Vladimir Vovk}

\begin{document}
\maketitle

\begin{abstract}
  This note continues study of exchangeability martingales,
  i.e., processes that are martingales under any exchangeable distribution for the observations.
  Such processes can be used for detecting violations of the IID assumption,
  which is commonly made in machine learning.
  Violations of the IID assumption are sometimes referred to as dataset shift,
  and dataset shift is sometimes subdivided into concept shift, covariate shift, etc.
  Our primary interest is in concept shift,
  but we will also discuss exchangeability martingales that decompose perfectly
  into two components one of which detects concept shift and the other detects what we call label shift.
  Our methods will be based on techniques of conformal prediction.

   The version of this paper at \url{http://alrw.net} (Working Paper 31)
   is updated most often.
\end{abstract}

\section{Introduction}

The most standard way of testing statistical hypotheses is batch testing:
we try to reject a given null hypothesis based on a batch of data.
The alternative approach of online testing (see, e.g., \cite{Shafer/Vovk:2019} or \cite{Shafer:arXiv1903})
consists in constructing a nonnegative process that is a martingale under the null hypothesis.
The ratio of the current value of such a process to its initial value can be interpreted
as the amount of evidence found against the null hypothesis.

The standard assumption in machine learning is the (general) IID assumption,
sometimes referred to (especially in older literature) as the assumption of randomness:
the observations are assumed to be independent and identically distributed,
but nothing is assumed about the probability measure generating a single observation.
Interestingly, there exist processes, \emph{exchangeability martingales},
that are martingales under the IID assumption;
they can be constructed (see, e.g., \cite[Section 7.1]{Vovk/etal:2005book} or \cite{\OCMXXIV})
using the method of conformal prediction \cite[Chapter 2]{Vovk/etal:2005book}.

Deviations from the IID assumption have become a popular topic of research in machine learning
under the name of dataset shift \cite{Candela/etal:2009,Moreno-Torres/etal:2012};
in my terminology I will follow mostly \cite{Moreno-Torres/etal:2012}.
Analysing general dataset shift is usually regarded as too challenging a problem,
and researchers concentrate on restricted versions,
with restrictions imposed on marginal or conditional probabilities
associated with the probability measure generating a single observation.
Different restrictions are appropriate for different kinds of learning problems.

In this note we consider problems of classification,
in which random observations $(X,Y)$ consist of objects $X$ and labels $Y$,
the latter taking a finite number of possible values.
We will be interested in \emph{$Y\to X$ domains}, in the terminology of \cite{Fawcett/Flach:2005},
in which the objects are causally dependent on the labels.
Under the IID assumption, the consecutive pairs $(X,Y)$ have the same probability distribution $P$.
There is a \emph{dataset shift} if $P$ in fact changes between observations.
Let us say that there is a \emph{label shift} if the marginal distribution $P_Y$ of $Y$ under $P$ changes.
Finally, there is a \emph{concept shift} if the conditional distribution $P_{X\mid Y}$ of $X$ given $Y$ changes.
Later in this note we will adopt a wider understanding of a label shift.

As an example, suppose we are interested in the differential diagnosis between cold, flu, and Covid-19
given a set of symptoms.
Under a pure label shift, the properties of the three diseases do not change (there is no concept shift),
and only their prevalence changes, perhaps due to epidemics and pandemics.
Under a concept shift, one or more of the diseases change leading to different symptoms.
Examples are new variants of Covid-19 and new strains of flu that appear every year.

In general, exchangeability martingales may detect both label shift and concept shift.
In some cases we might not be interested in label shift and only be interested in concept shift
(or, perhaps less commonly, vice versa).
The goal of this note is to develop and start investigating exchangeability martingales
targeting only concept shift.
It would be ideal to decompose the amount of evidence found by an exchangeability martingale for dataset shift
into two components, one reflecting the amount of evidence found for concept shift
and the other reflecting the amount of evidence found for label shift.
Such decomposable martingales are our secondary object of study.

New exchangeability martingales and their simple theoretical properties will be the topic of Section~\ref{sec:theory},
and in Section~\ref{sec:experiments} they will be applied to the well-known USPS dataset.
The preliminary results reported in the latter section suggest that the exchangeability martingales
constructed for this dataset in \cite[Section 7.1]{Vovk/etal:2005book} are dominated (and greatly improved)
by an exchangeability martingale decomposable
into a product of an exchangeability martingale for detecting concept shift
and an exchangeability martingale for detecting label shift.

The most obvious application of exchangeability martingales
is to help in deciding when to retrain predictors,
as discussed in \cite{\OCMXXIV}.
We should be particularly worried about the changes that invalidate ROC analysis,
which is the case of concept shift in a $Y\to X$ domain
\cite{Webb/Ting:2005,Fawcett/Flach:2005}.
Our exchangeability martingales for concept shift are designed to detect
such dangerous changes.

In the context of conformal prediction,
concept shift in $Y\to X$ domains requires retraining label-conditional predictors
\cite[Section 4.5]{Vovk/etal:2005book}.
For connection between label-conditional predictors and ROC analysis,
see \cite[Section 2.7]{Bala/etal:2014}.

\section{Theory}
\label{sec:theory}

For a detailed review of conformal prediction see, e.g., \cite{Vovk/etal:2005book},
but in this section I will mainly follow \cite[Chapters 1 and 2]{Bala/etal:2014}
(for the generation of conformal p-values)
and \cite{\OCMXXIV} (for gambling against those p-values).

As mentioned earlier, we consider \emph{observations} $z=(x,y)$ that consist of two components,
the \emph{object} $x$ and the \emph{label} $y$.
Let $\mathbf{X}$ be the measurable space of all possible objects,
and $\mathbf{Y}$ be the set of all possible labels.
Set $\mathbf{Z}:=\mathbf{X}\times\mathbf{Y}$; this is our \emph{observation space}.
We are interested in classification and so always assume $\left|\mathbf{Y}\right|<\infty$;
$\mathbf{Y}$ is always equipped with the discrete $\sigma$-algebra.

A \emph{conformity measure} $A$ is a function that maps
any finite sequence $(z_1,\dots,z_n)\in\mathbf{Z}^n$ of observations of any length $n\in\{1,2,\dots\}$
to a sequence $(\alpha_1,\dots,\alpha_n)\in\R^n$ of real numbers of the same length
that is \emph{equivariant} in the following sense:
for any $n\in\{1,2,\dots\}$, any permutation $\pi:\{1,\dots,n\}\to\{1,\dots,n\}$,
and any sequences $(z_1,\dots,z_n)\in\mathbf{Z}^n$ and $(\alpha,\dots,\alpha_n)\in\R^n$,
\begin{equation}\label{eq:equivariance}
  \left(
    \alpha_1,\dots,\alpha_n
  \right)
  =
  A
  \left(
    z_1,\dots,z_n
  \right)
  \Longrightarrow
  \left(
    \alpha_{\pi(1)},\dots,\alpha_{\pi(n)}
  \right)
  =
  A
  \left(
    z_{\pi(1)},\dots,z_{\pi(n)}
  \right).
\end{equation}

In our experiments in Section~\ref{sec:experiments} we will only use conformity measures,
but in theory we are also interested in the following generalization.
A \emph{label-conditional conformity measure} $A$ is a function that maps
any finite sequence $(z_1,\dots,z_n)\in\mathbf{Z}^n$ of observations of any length $n\in\{1,2,\dots\}$
to a sequence $(\alpha_1,\dots,\alpha_n)\in\R^n$ of real numbers of the same length
that is \emph{label-conditionally equivariant}:
for any $n\in\{1,2,\dots\}$, any permutation $\pi:\{1,\dots,n\}\to\{1,\dots,n\}$,
and any sequences $(z_1,\dots,z_n)=((x_1,y_1),\dots,(x_n,y_n))\in\mathbf{Z}^n$ and $(\alpha,\dots,\alpha_n)\in\R^n$,
\[
  \begin{rcases}
    \quad
    y_1 = y_{\pi(1)}, \dots,
    y_n = y_{\pi(n)}\\
    \left(
      \alpha_1,\dots,\alpha_n
    \right)
    =
    A
    \left(
      z_1,\dots,z_n
    \right)
  \end{rcases}
  \Longrightarrow
  \left(
    \alpha_{\pi(1)},\dots,\alpha_{\pi(n)}
  \right)
  =
  A
  \left(
    z_{\pi(1)},\dots,z_{\pi(n)}
  \right).
\]
In other words, we only require \eqref{eq:equivariance} to hold for the permutations that leave the labels intact.

The \emph{label-conditional conformal transducer} associated with a label-conditional conformity measure $A$
is the function $p$ defined by
\begin{equation}\label{eq:p}
  p(z_1,\dots,z_n,\tau)
  :=
  \frac
  {
    \left|\left\{
      i : y_i=y_n \land \alpha_i<\alpha_n
    \right\}\right|
    +
    \tau
    \left|\left\{
      i : y_i=y_n \land \alpha_i=\alpha_n
    \right\}\right|
  }
  {
    \left|\left\{
      i : y_i=y_n
    \right\}\right|
  },
\end{equation}
where $i$ ranges over $1,\dots,n$,
$z_i=(x_i,y_i)$ for all $i\in\{1,\dots,n\}$,
\begin{equation}\label{eq:alpha}
  \left(
    \alpha_1,\dots,\alpha_n
  \right)
  :=
  A
  \left(
    z_1,\dots,z_n
  \right),
\end{equation}
and $\tau\in[0,1]$.
The values~\eqref{eq:p} will be referred to as \emph{p-values}.
If the label-conditional conformity measure $A$ is in fact a conformity measure,
we will say that the label-conditional conformal transducer $p$ associated with it is \emph{simple}.

Let $Z_1,Z_2,\dots$ be a sequence of random observations,
i.e., random elements whose domain is a fixed probability space with probability measure $\Prob$
and which take values in the observation space $\mathbf{Z}$.
Each random observation $Z_n$ is a pair $Z_n=(X_n,Y_n)$,
where $X_n$ is a random object and $Y_n$ is a random label.

Let us say that the random sequence of observations $Z_1,Z_2,\dots$ is \emph{label-conditional exchangeable}
if, for any $n\in\{1,2,\dots\}$, any sequence $(y_1,\dots,y_n)\in\mathbf{Y}^n$,
any sequence of measurable sets $E_1,\dots,E_n$ in $\mathbf{X}$,
and any permutation $\pi:\{1,\dots,n\}\to\{1,\dots,n\}$,
\begin{multline*}
  y_1 = y_{\pi(1)}, \dots,
  y_n = y_{\pi(n)}\\
  \Longrightarrow
  \Prob
  \left(
    Y_1=y_1,\dots,Y_n=y_n,
    X_1\in E_1,\dots,X_n\in E_n
  \right)\\
  =
  \Prob
  \left(
    Y_1=y_1,\dots,Y_n=y_n,
    X_{\pi(1)}\in E_1,\dots,X_{\pi(n)}\in E_n
  \right).
\end{multline*}
This is an instance of de Finetti's \cite{deFinetti:1938} notion of partial exchangeability.
The sequence $Z_1,Z_2,\dots$ is \emph{exchangeable} if, for any $n\in\{1,2,\dots\}$,
any sequence of measurable sets $E_1,\dots,E_n$ in $\mathbf{Z}$,
and any permutation $\pi:\{1,\dots,n\}\to\{1,\dots,n\}$,
\[
  \Prob
  \left(
    Z_1\in E_1,\dots,Z_n\in E_n
  \right)
  =
  \Prob
  \left(
    Z_{\pi(1)}\in E_1,\dots,Z_{\pi(n)}\in E_n
  \right).
\]
Of course, exchangeability is a stronger property than label-conditional exchangeability.

\begin{proposition}\label{prop:label-conditional}
  If the sequence of random observations $Z_1,Z_2,\dots$ is label-conditional exchangeable,
  $(\tau_1,\tau_2,\dots)$ is an independent sequence of independent random variables each distributed uniformly in $[0,1]$,
  and $p$ is a label-conditional conformal transducer,
  the sequence of random p-values
  \begin{equation}\label{eq:P}
    P_n
    :=
    p(Z_1,\dots,Z_n,\tau_n),
    \quad
    n=1,2,\dots,
  \end{equation}
  is distributed uniformly in $[0,1]^{\infty}$.
\end{proposition}

For a proof of Proposition~\ref{prop:label-conditional},
see \cite[Section 8.7]{Vovk/etal:2005book}
(Proposition~\ref{prop:label-conditional} is a special case of Theorem~8.1 in \cite{Vovk/etal:2005book}).

If $Z$ is a measurable space, $Z^*$ stands for the set of all finite sequences of elements of $Z$
(equipped with the natural $\sigma$-algebra).
It includes the empty sequence $\Box$.
A \emph{betting martingale} is a measurable function $F:[0,1]^*\to[0,\infty]$
such that $F(\Box)=1$ and,
for each sequence $(u_1,\dots,u_{n-1})\in[0,1]^{n-1}$ for any $n\in\{1,2,\dots\}$,
\begin{equation}\label{eq:betting-martingale}
  \int_0^1
  F(u_1,\dots,u_{n-1},u)
  \dd u
  =
  F(u_1,\dots,u_{n-1}).
\end{equation}
(The three unusual features of this definition are that betting martingales are required to be nonnegative,
start from 1,
and are allowed to take value $\infty$.)
The \emph{test martingale} associated with the betting martingale $F$ and a sequence $(P_1,P_2,\dots)$
uniformly distributed in $[0,1]^{\infty}$
(the \emph{input p-values})
is the sequence of random variables
\begin{equation}\label{eq:S_n}
  S_n
  =
  F(P_1,\dots,P_n),
  \quad
  n=0,1,\dots.
\end{equation}
The sequence $(S_n)_{n=0,1,\dots}$ is a nonnegative martingale,
in the usual sense of probability theory \cite[Definition~7.1.1]{Shiryaev:2019},
in its own filtration $\FFF_n:=\sigma(S_1,\dots,S_n)$
or the filtration $\FFF_n:=\sigma(P_1,\dots,P_n)$ generated by the input p-values.
Intuitively, this martingale describes the evolution of the capital of a player
who gambles against the hypothesis that the input p-values are distributed uniformly and independently.

In this note we will be interested in three classes of martingales.
The \emph{label-conditional conformal martingales} are defined
as the test martingales associated with any betting martingale $F$
and a sequence $(P_1,P_2,\dots)$ defined by \eqref{eq:P}
(under the conditions of Proposition~\ref{prop:label-conditional})
as the input p-values.

Label-conditional conformal martingales are main topic of this note.
They detect concept shift.
It was shown, once again, in \cite[Section~7.1]{Vovk/etal:2005book}
that the USPS dataset is non-exchangeable,
and in Section~\ref{sec:experiments} we will explore sources of this lack of exchangeability.

\begin{remark}
  It is important that our exchangeability martingales for detecting concept shift
  can be used in situations where the labels are so far from being IID
  that it would be unusual to talk about label shift.
  Discussion of label shift usually presuppose at least approximate independence of labels.
  Suppose a sequence of hand-written characters $x_1,x_2,\dots$ comes from a user writing a letter.
  The objects $x_n$ are matrices of pixels and the corresponding labels $y_n$ take values in the set $\{a,b,\dots\}$.
  Different instances of the same character, say ``a'', may well be exchangeable among themselves
  (even conditionally on knowing the full text of the letter),
  whereas the text itself will be far from IID;
  for example, ``q'' will be almost invariably followed by ``u'' if the letter is in English.
  For discussions of such partial exchangeability,
  see, e.g., \cite{deFinetti:1938}, \cite{Ryabko:2006}, and \cite[Section 8.4]{Vovk/etal:2005book}.
\end{remark}

In the rest of this section we will look for possible explanations of the difference
between the amount of evidence found against concept shift and against exchangeability.
We will see that in some situation the amount of evidence found against exchangeability decomposes into two components:
\begin{itemize}
\item
  the amount of evidence found for concept shift;
\item
  the amount of evidence found for label shift.
\end{itemize}
In these situations the second component can be said to explain the difference.

A \emph{label conformity measure} $A$ is a conformity measure that satisfies, additionally,
the following property:
for any finite sequence $(z_1,\dots,z_n)\in\mathbf{Z}^n$ of observations of any length $n\in\{1,2,\dots\}$,
any sequence $(\alpha_1,\dots,\alpha_n)\in\R^n$ of real numbers of the same length,
and any $i,j\in\{1,\dots,n\}$,
\begin{equation}\label{eq:label-invariance}
  \left.
  \begin{aligned}
    y_i &= y_j\\
    \left(
      \alpha_1,\dots,\alpha_n
    \right)
    &=
    A
    \left(
      z_1,\dots,z_n
    \right)
  \end{aligned}
  \right\}
  \Longrightarrow
  \alpha_i = \alpha_j,
\end{equation}
where $y_i$ and $y_j$ are the labels in $z_i$ and $z_j$, respectively.
In other words, it assigns conformity scores only to the labels rather than to the full observations.
(Notice that the requirement of equivariance only ensures \eqref{eq:label-invariance}
with ``$z_i=z_j$'' in place of ``$y_i=y_j$''.)
The \emph{conformal transducer} associated with a conformity measure $A$ outputs the p-values
\begin{equation}\label{eq:p-prime}
  p'(z_1,\dots,z_n,\tau)
  :=
  \frac
  {
    \left|\left\{
      i : \alpha_i<\alpha_n
    \right\}\right|
    +
    \tau
    \left|\left\{
      \alpha_i=\alpha_n
    \right\}\right|
  }
  {n},
\end{equation}
where $i\in\{1,\dots,n\}$, $\alpha_1,\dots,\alpha_n$ are defined by \eqref{eq:alpha},
and $\tau\in[0,1]$.
We will say that $p'$ is a \emph{label conformal transducer} if $A$ is a label conformity measure.

Our method of decomposing exchangeability martingales will be based on the following result
(version of Theorem~8.1 in \cite{Vovk/etal:2005book}).
Its proof is given in Appendix~\ref{app:proof}.

\begin{theorem}\label{thm:interleaved}
  If the sequence of random observations $Z_1,Z_2,\dots$ is exchangeable,
  $(\tau_1,\tau_2,\dots)$ and $(\tau'_1,\tau'_2,\dots)$ are independent
  (between themselves and of the observations)
  sequences distributed uniformly in $[0,1]^{\infty}$,
  $p$ is a simple label-conditional conformal transducer,
  and $p'$ is a label conformal transducer,
  the interleaved sequence of random p-values $P_1,P'_1,P_2,P'_2,\dots$, where
  \[
    P_n
    :=
    p(Z_1,\dots,Z_n,\tau_n),
    \quad
    P'_n
    :=
    p'(Z_1,\dots,Z_n,\tau'_n),
  \]
  is distributed uniformly in $[0,1]^{\infty}$.
\end{theorem}

A \emph{conformal martingale} is defined to be the test martingale associated
(via \eqref{eq:S_n}, where $F$ is a betting martingale)
with a conformal transducer.
If the underlying conformity measure is a label conformity measure,
the conformal martingale will be called a \emph{label conformal martingale}.

We will say that a label-conditional conformal martingale is \emph{simple}
if its underlying label-conditional conformal transducer is simple.

Having the stream of random p-values $P_1,P'_1,P_2,P'_2,\dots$ produced as in Theorem~\ref{thm:interleaved},
we can define two derivative exchangeability martingales:
a label-conditional conformal martingale associated with $P_1,P_2,\dots$
and a label conformal martingale associated with $P'_1,P'_2,\dots$.
(There are no restrictions on the underlying betting martingales.)

\begin{corollary}\label{cor:product}
  The product of a simple label-conditional conformal martingale and a label conformal martingale
  with independent randomizations (i.e., their sequences of random numbers $\tau$)
  is an exchangeability martingale.
\end{corollary}

Such product exchangeability martingales decompose perfectly
into components for detecting concept shift and label shift.
For a short proof of this corollary, see Appendix~\ref{app:proof-corollary}.

\section{Experiments}
\label{sec:experiments}

The dataset used in our experiment is the well-known USPS dataset of hand-written digits
\cite[Appendix B.1]{Vovk/etal:2005book},
which is known to be non-exchangeable.
The objects $x_n$ are $16\times16$ matrices with entries in $[-1,1]$
(representing pixel intensities),
and the labels $y_n$ are elements of $\{0,\dots,9\}$;
overall there are 9298 labelled images
(obtained by merging the original training set of 7291 and test set of 2007).
This dataset is clearly in the $Y\to X$ domain
(the writer's intention causes the resulting matrix of pixels,
not vice versa).

\begin{figure}
  \begin{center}
    \includegraphics[width=0.49\textwidth]{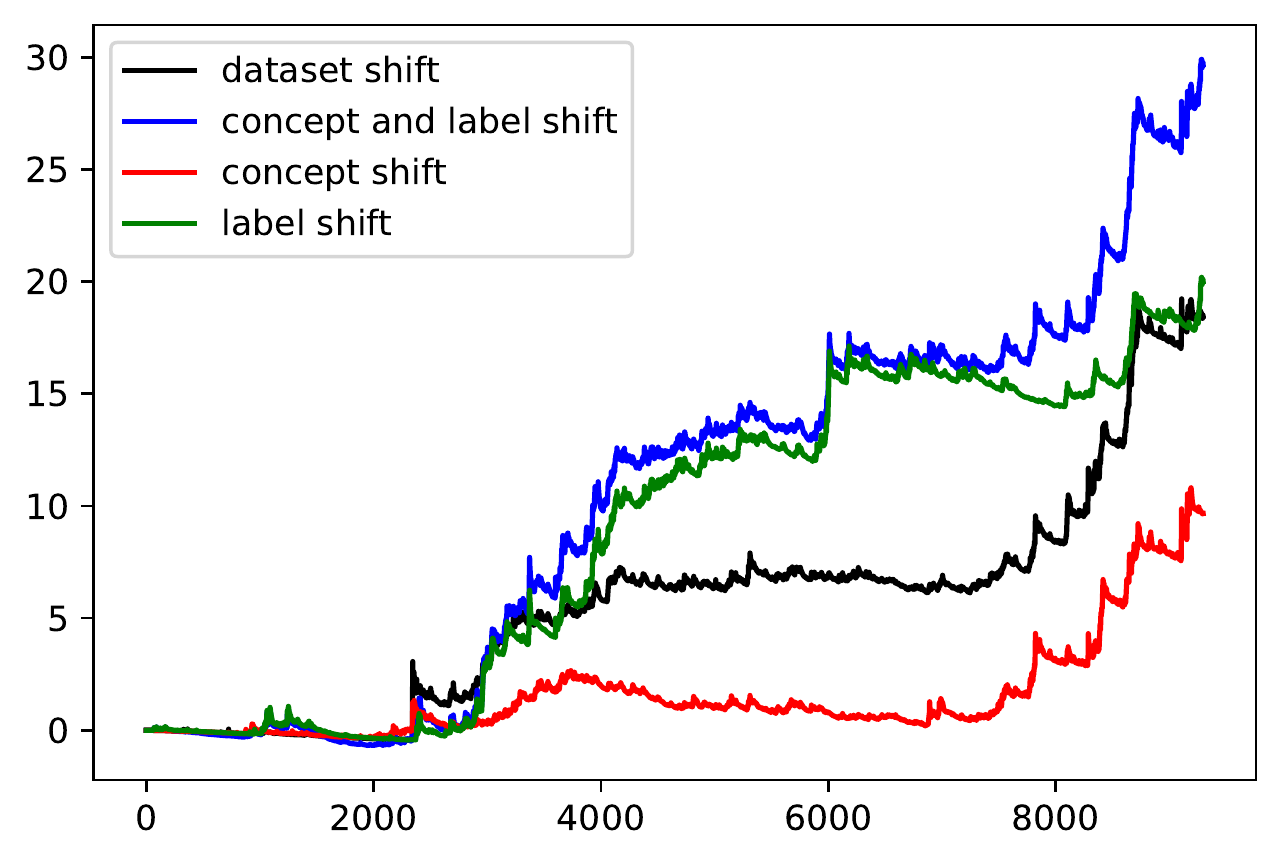}
    \includegraphics[width=0.49\textwidth]{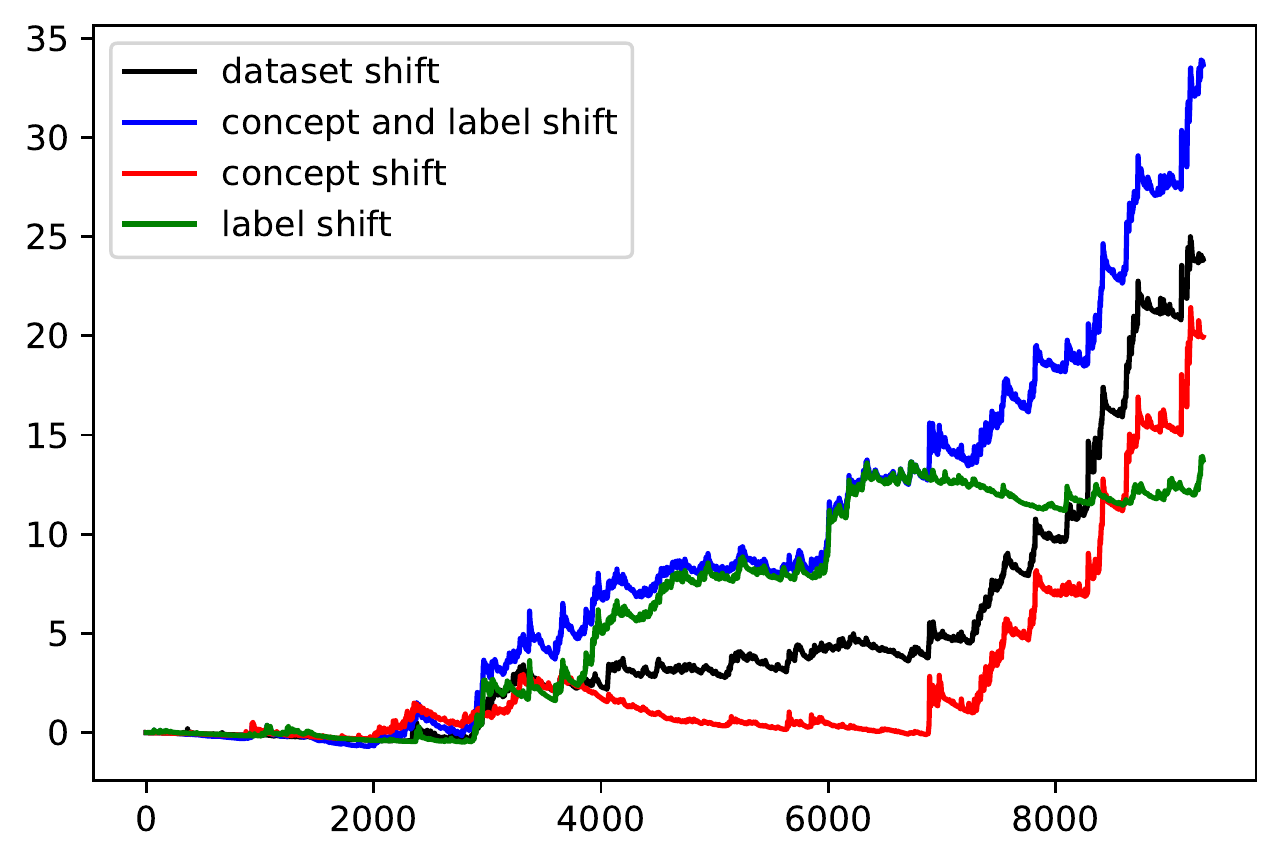}
  \end{center}
  \caption{Four exchangeability martingales for the ratio conformity measure (left panel)
    and its modification described in text (right panel)}
  \label{fig:ratio}
\end{figure}

Online methods for testing the exchangeability of the USPS dataset
are described in \cite[Section 7.1]{Vovk/etal:2005book}.
The best result reported there \cite[Figure 7.6]{Vovk/etal:2005book}
is given as the black line in the left panel of Figure~\ref{fig:ratio}.
It plots $n\in\{0,\dots,9298\}$ vs the value $S_n$ of a conformal martingale with initial value 1
after processing the first $n$ observations.
The values of $S_n$ are given on the log (base 10) scale.
The final value $S_{9298}$ exceeds $10^{18}$.
The martingale is randomized, but its trajectory does not depend much on the seed used in the random number generator
(and this will be true for all other conformal martingales discussed in this note).

The conformity measure used in \cite[Figure 7.6]{Vovk/etal:2005book} is of the nearest-neighbour type:
namely, the conformity score $\alpha_i$ of the $i$th observation $(x_i,y_i)$ in a sequence
$(x_1,y_1),\dots,(x_n,y_n)$ is defined as
\begin{equation}\label{eq:ratio}
  \alpha_i
  :=
  \frac
  {
    \min_{j: y_j \ne y_i}
    \left\|x_i-x_j\right\|
  }
  {
    \min_{j\ne i: y_j = y_i}
    \left\|x_i-x_j\right\|
  },
\end{equation}
where $\left\|\dots\right\|$ is Euclidean norm.
(Using the tangent distance in place of the Euclidean distance $\left\|x-x'\right\|$ leads to similar results,
for all experiments reported in this note,
unlike the batch experiments in \cite[Section 2]{\OCMXXIV}.)

The conformity score \eqref{eq:ratio} is the ratio of the (nearest) distance to another class
to the distance to the same class (excluding the current observation).
This conformity measure will be referred to as the \emph{ratio} conformity measure.
Later we will also be interested in modifications of this conformity measure.

The betting martingale used in all our experiments is the \emph{Sleepy Jumper},
as described in \cite[Section 7.1]{Vovk/etal:2005book}.
I will not repeat the definition here, and only mention that it involves two parameters,
$R=0.01$ and $J=0.001$.
(Inevitably, there is some element of data snooping here,
since these values were chosen because of their reasonable performance on the USPS dataset,
but it is limited by the use of round figures.)
Only these values of parameters will be used in this note.

Each of the four exchangeability martingales in Figure~\ref{fig:ratio}
apart from the product (the blue martingale) is determined by three components:
\begin{itemize}
\item
  the underlying conformity measure,
  which is either \eqref{eq:ratio} or one of its modifications;
\item
  the transducer,
  which is either the label-conditional conformal transducer \eqref{eq:p}
  or the conformal transducer \eqref{eq:p-prime};
  feeding the conformity measure of the previous item into this transducer
  we obtain a sequence of p-values;
\item
  the betting martingale $F$,
  which in this note is always the Sleepy Jumper;
  we feed the p-values resulting from the previous item into $F$,
  as per \eqref{eq:S_n}.
\end{itemize}
The black martingale in the left panel of Figure~\ref{fig:ratio}
uses the conformity measure \eqref{eq:ratio},
the conformal transducer \eqref{eq:p-prime},
and the Sleepy Jumper.

The black martingale may detect any deviations from exchangeability,
but in this note we are particularly interested in concept shift.
In our current context,
concept shift means that, for some reason, the same digit (such as ``0'') starts looking different;
perhaps people start writing digits differently, or the digits are scanned with different equipment.
To detect concept shift,
we use the same conformity measure \eqref{eq:ratio},
but feed it into the label-conditional conformal transducer \eqref{eq:p};
the resulting sequence of p-values is fed into the Sleepy Jumper, as usual.
The resulting test martingale is shown in red in the left panel of Figure~\ref{fig:ratio}.
Its final value, of the order of magnitude $10^{10}$,
is much less impressive than the final value of the black martingale,
and the red martingale starts its climb towards its final value
over the original test set (the last 2007 observations).

There is, of course, another reason why exchangeability may be violated:
we may have label shift.
To detect it, we use the label conformity measure
that assigns the conformity score
\begin{equation}\label{eq:p-invariant}
  \alpha'_i
  :=
  \frac
  {\sum_{j: y_j=y_i} \alpha_i}
  {\left|\left\{j: y_j=y_i\right\}\right|}
\end{equation}
to the $i$th observation $(x_i,y_i)$ in a sequence $(x_1,y_1),\dots,(x_n,y_n)$.
In other words, we average the conformity scores for each class
to ensure the requirement of invariance \eqref{eq:label-invariance}.

The label conformal martingale obtained by applying the Sleepy Jumper
to the p-values produced by the label conformal transducer \eqref{eq:p-prime}
applied to the conformity scores \eqref{eq:p-invariant}
is shown as the green line in the left panel of Figure~\ref{fig:ratio}.
It is interesting that, despite the invariance restriction,
the final value of the green martingale is even greater than the final value of the black martingale.
The dataset shift can be explained by just the label shift.

According to Corollary~\ref{cor:product},
the product of a label-conditional conformal martingale and a label conformal martingale
is still an exchangeability martingale.
The product is shown as the blue line in the left panel of Figure~\ref{fig:ratio}.
By construction, the blue martingale is perfectly decomposable.
Its final value greatly exceeds the previous record for the USPS dataset
(the final value achieved by the black martingale).

\begin{remark}
  Corollary~\ref{cor:product} has an important condition,
  ``with independent randomizations''.
  It is ignored in this version of the note,
  where the seed of the random number generator is always set to 1.
  Corollary~\ref{cor:product} remains applicable to a high degree of approximation
  since the dependence on the seed of the random number generator is weak.
  This somewhat cavalier approach is likely to change when the Python code for this note
  is rewritten to comply with the recent changes in NumPy random number generation \cite{Kern:2018}.
\end{remark}

The blue exchangeability martingale,
on the one hand, almost dominates the black martingale over the USPS dataset
(namely, it dominates after approximately 3000 observations)
and, on the other hand,
decomposes into a product of exchangeability martingales for detecting concept shift
and for detecting label shift.
Therefore, the red and green pair in the left panel of Figure~\ref{fig:ratio} appears to be a significant improvement
over the black martingale.

For other conformity measures we will often obtain results that are qualitatively different.
For example, squaring the denominator of \eqref{eq:ratio} will lead to the right panel of Figure~\ref{fig:ratio}.
The performance of the exchangeability martingale for detection of concept shift greatly improves
over the original test set,
but the price to pay is deterioration in the performance of the exchangeability martingale
for detection of label shift.

\begin{figure}
  \begin{center}
    \includegraphics[width=0.49\textwidth]{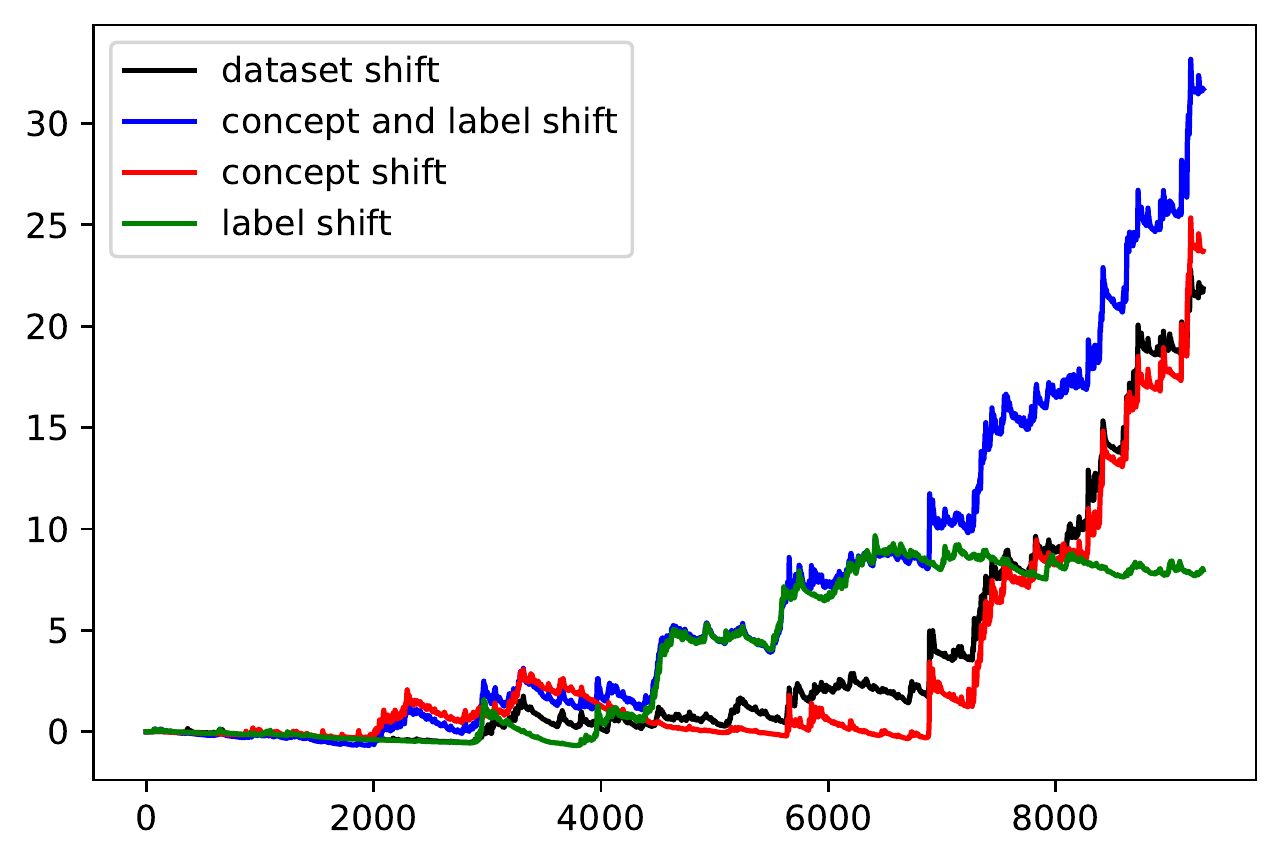}
    \includegraphics[width=0.49\textwidth]{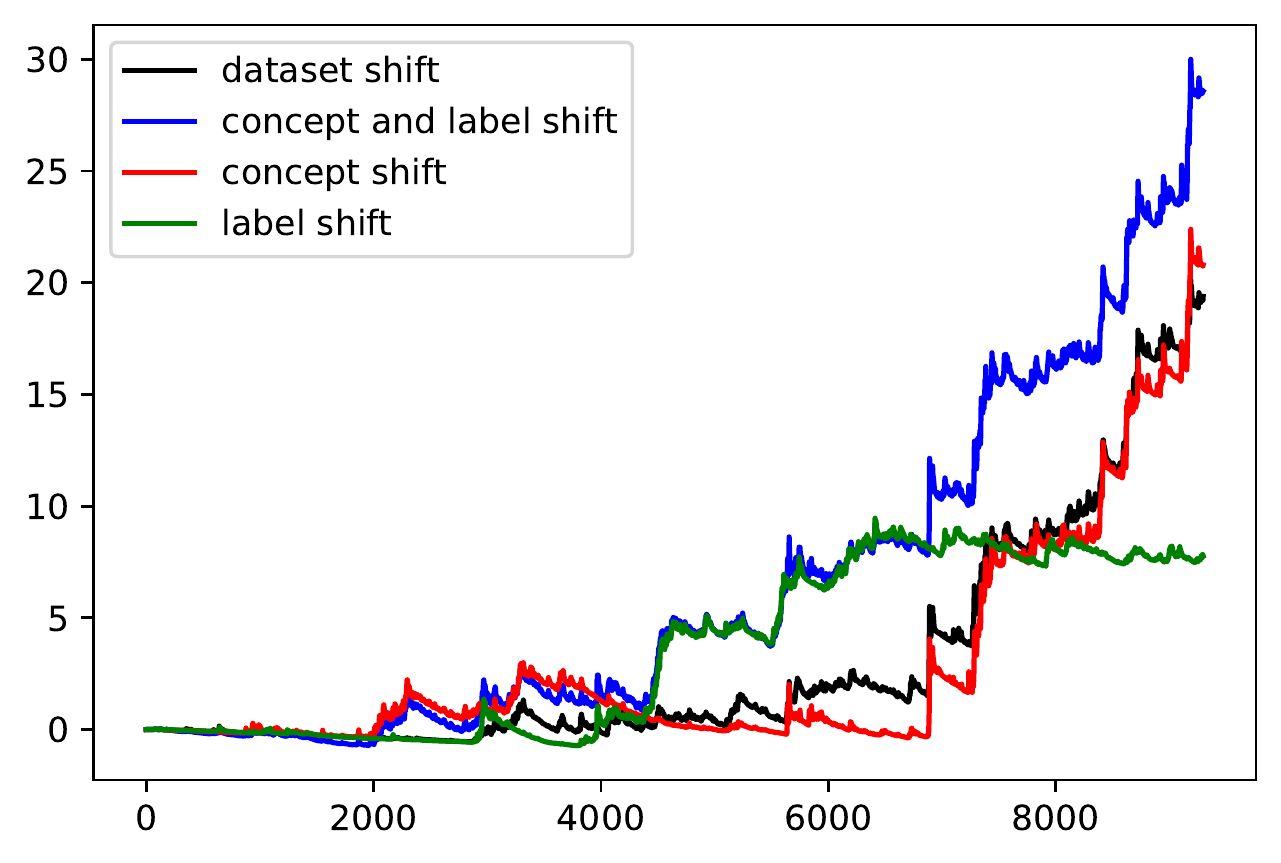}
  \end{center}
  \caption{The exchangeability martingales for the same-class conformity measure (left panel)
    and the nearest-object conformity measure (right panel)}
  \label{fig:same}
\end{figure}

A more radical modification of \eqref{eq:ratio}
is obtained by replacing the numerator of \eqref{eq:ratio} by 1;
let us call it the \emph{same-class conformity measure}.
The resulting exchangeability martingales
(still using the Sleepy Jumper as the betting martingale)
are shown in the left panel of Figure~\ref{fig:same}.
Detection of concept shift becomes even more successful,
and detection of label shift suffers further.

In the definition of the same-class conformity measure,
it is tempting to replace the distance to the nearest neighbour to the same class
(the denominator of \eqref{eq:ratio}) by the distance to the nearest neighbour;
after all, the nearest neighbour can be expected to be of the same class.
This, however, leads to slight deterioration in the final value of the red martingale
(which is our primary interest);
see the right panel of Figure~\ref{fig:same},
where this modification is referred to as the \emph{nearest-object conformity measure}.

Of course, we do not have to use the same conformity measure when combining
the red and green martingales in Figures~\ref{fig:ratio} and~\ref{fig:same}:
Theorem~\ref{thm:interleaved} does not impose any conditions on the conformity measures
giving rise to $p$ and $p'$.
This allows us to obtain much larger final values for a valid exchangeability martingale.
For example, combining the green martingale in the left panel of Figure~\ref{fig:ratio}
and the red martingale in the left panel of Figure~\ref{fig:same},
we obtain an exchangeability martingale that turns 1 into what looks about $10^{45}$
(in fact, $5.16\times10^{43}$ in this experiment).

\section{Conclusion}

We have seen that the existing methods of constructing exchangeability martingales
can be adapted to detecting concept shift.
Perfectly decomposable exchangeability martingales turned out to be surprisingly successful
on the USPS dataset of handwritten digits.

This note concentrated on concept shift in $Y\to X$ classification domains.
It is clear, however, that the same methods are applicable, verbatim,
when the observations $z_i$ take values in any measurable space
and $y_i$ are no longer the labels but defined as $f(z_i)$ for a function $f$ taking finitely many values.
For example, $y_i$ can be an important feature of the object in $z_i$ that we do not wish to model,
but we wish our analysis to be conditional on it
(e.g., $y_i\in\{\text{male},\text{female}\}$ can be a feature).

\subsection*{Acknowledgments}

This work has been supported by Amazon
(project ``Conformal martingales for change-point detection'') and Stena Line.

\appendix
\section{Proof of Theorem~\protect\ref{thm:interleaved}}
\label{app:proof}

It suffices to prove, for a fixed horizon $N\in\{1,2,\dots\}$,
that the random p-values $P'_1,P_1,\dots,P'_N,P_N$ are distributed independently and uniformly in $[0,1]$
(see, e.g., \cite[Section 8.2]{Vovk/etal:2005book}).
Let us fix such an $N$.

The rest of this appendix is a modification of \cite[Section 8.7]{Vovk/etal:2005book}.
First an informal argument.
Imagine that the data sequence $Z_1,\dots,Z_n$ is generated in two steps:
first a random multiset $\lbag Z_1,\dots,Z_n\rbag$ and then its random ordering.
Already the second step ensures that $(P_1,P'_1,\dots,P_N,P'_N)$ are distributed uniformly in $[0,1]^{2N}$
(even conditionally on $\lbag Z_1,\dots,Z_n\rbag$).
This can be demonstrated using the following backward argument.
Ignoring borderline effects, $P'_N$ is uniformly distributed in $[0,1]$ (at least approximately).
When $Y_N$ is disclosed, $P'_N$ will be settled.
Given what we already know, the distribution of $P_N$ will be uniform.
When $X_N$ is disclosed, $P_N$ will be settled.
Now the distribution of $P'_{N-1}$ given what we already know is uniform, etc.

For the formal proof, we will need the following $\sigma$-algebras.
Let $\GGG_n$, $n=0,\dots,N$, be the $\sigma$-algebra
\[
  \GGG_n
  :=
  \sigma
  \left(
    \lbag Z_1,\dots,Z_n\rbag,
    Z_{n+1},\tau_{n+1},\tau'_{n+1},\dots,Z_{N},\tau_{N},\tau'_{N}
  \right)
\]
generated by the multiset $\lbag Z_1,\dots,Z_n\rbag$
and the other random elements listed in the parentheses.
Let $\GGG'_n$, $n=1,\dots,N$, be the $\sigma$-algebra $\sigma(\GGG_n,Y_n,\tau'_n)$
generated by $\GGG_n$, the label $Y_n$ of the $n$th observation, and the random number $\tau'_n$.

The following two lemmas (analogues of \cite[Lemma 8.8]{Vovk/etal:2005book}) say that
 \begin{empheq}[box=\ovalbox]{align*}
                      &P'_N \hspace{5mm} P_N \hspace{7mm}     P'_{N-1} \hspace{4mm} \dots \hspace{4.5mm} P_2 \hspace{4.5mm} P'_1 \hspace{5mm} P_1 \\
    \GGG_N \subseteq{}&\GGG'_N \subseteq \GGG_{N-1} \subseteq \GGG'_{N-1} \subseteq \dots \subseteq      \GGG_1 \subseteq   \GGG'_1 \subseteq \GGG_0
\end{empheq}
is a stochastic sequence essentially in the usual sense of probability theory
\cite[Section~7.1.2]{Shiryaev:2019}:
in the second row we have a finite filtration,
and the random variables in the first row are measurable w.r.\ to the $\sigma$-algebras directly below them.

\begin{lemma}\label{lem:prime-measurable}
  For any trial $n=1,\dots,N$,
  $P'_n$ is $\GGG'_n$-measurable.
\end{lemma}

\begin{proof}
  The random multiset of conformity scores of $Z_1,\dots,Z_n$ is $\GGG_n$-mea\-sur\-able,
  and so, according to the definition \eqref{eq:p-prime} and the invariance requirement \eqref{eq:label-invariance},
  $P'_n$ is $\GGG'_n$-measurable.
\end{proof}

\begin{lemma}\label{lem:measurable}
  For any trial $n=1,\dots,N$,
  $P_n$ is $\GGG_{n-1}$-measurable.
\end{lemma}

\begin{proof}
  This follows from the definition \eqref{eq:p}
  and our requirement that the label-conditional conformal transducer $p$ should be simple.
\end{proof}

We will also need the following analogues of \cite[Lemma 8.7]{Vovk/etal:2005book}.
As in \cite{Vovk/etal:2005book}, $\Expect_{\FFF}$ stands for the conditional expectation
w.r.\ to a $\sigma$-algebra $\FFF$.

\begin{lemma}\label{lem:GGG-prime}
  For any trial $n=1,\dots,N$ and any $\epsilon\in[0,1]$,
  \[
    \Prob_{\GGG'_n}
    \left\{
      P_n \le \epsilon
    \right\}
    =\epsilon.
  \]
\end{lemma}

\begin{proof}
  Follow the proof of \cite[Lemma 8.7]{Vovk/etal:2005book}.
\end{proof}

\begin{lemma}\label{lem:GGG}
  For any trial $n=1,\dots,N$ and any $\epsilon\in[0,1]$,
  \[
    \Prob_{\GGG_n}
    \left\{
      P'_n \le \epsilon
    \right\}
    =\epsilon.
  \]
\end{lemma}

\begin{proof}
  Follow the proof of \cite[Lemma 8.7]{Vovk/etal:2005book}.
\end{proof}

Let us now prove the following double sequence of equalities:
\begin{multline}\label{eq:GGG-prime-full}
  \Prob_{\GGG'_n}
  \left\{
    P_n \le \epsilon_n,
    P'_{n-1} \le \epsilon'_{n-1},
    P_{n-1} \le \epsilon_{n-1},
    \dots,
    P'_1 \le \epsilon'_1,
    P_1 \le \epsilon_1
  \right\}\\
  =
  \epsilon_n
  \epsilon'_{n-1}
  \epsilon_{n-1}
  \dots
  \epsilon'_1
  \epsilon_1
\end{multline}
and
\begin{equation}\label{eq:GGG-full}
  \Prob_{\GGG_n}
  \left\{
    P'_n \le \epsilon'_n,
    P_n \le \epsilon_n,
    \dots,
    P'_1 \le \epsilon'_1,
    P_1 \le \epsilon_1
  \right\}
  =
  \epsilon'_n
  \epsilon_n
  \dots
  \epsilon'_1
  \epsilon_1.
\end{equation}
We will use induction arranging these equalities into a single sequence:
the equality for $\Prob_{\GGG'_1}$,
the equality for $\Prob_{\GGG_1}$,
the equality for $\Prob_{\GGG'_2}$,
the equality for $\Prob_{\GGG_2}$, etc.
The first of these equalities is a special case of Lemma~\ref{lem:GGG-prime}.
When proving any other of these equalities,
we will assume that all the previous equalities are true.

The equality for $\Prob_{\GGG_n}$, $n\in\{1,\dots,N\}$, follows from
\begin{multline*}
  \Prob_{\GGG_n}
  \left\{
    P'_n \le \epsilon'_n,
    P_n \le \epsilon_n,
    \dots,
    P'_1 \le \epsilon'_1,
    P_1 \le \epsilon_1
  \right\}\\
  =
  \Expect_{\GGG_n}
  \left(
    \Expect_{\GGG'_n}
    \left(
      1_{P'_n \le \epsilon'_n}
      1_{P_n \le \epsilon_n}
      \dots
      1_{P'_1 \le \epsilon'_1}
      1_{P_1 \le \epsilon_1}
    \right)
  \right)\\
  =
  \Expect_{\GGG_n}
  \left(
    1_{P'_n \le \epsilon'_n}
    \Expect_{\GGG'_n}
    \left(
      1_{P_n \le \epsilon_n}
      \dots
      1_{P'_1 \le \epsilon'_1}
      1_{P_1 \le \epsilon_1}
    \right)
  \right)\\
  =
  \Expect_{\GGG_n}
  \left(
    1_{P'_n \le \epsilon'_n}
    \epsilon_n
    \dots
    \epsilon'_1
    \epsilon_1
  \right)
  =
  \epsilon'_n
  \epsilon_n
  \dots
  \epsilon'_1
  \epsilon_1.
\end{multline*}
The first equality is just the tower property of conditional expectations.
The second equality follows from Lemma~\ref{lem:prime-measurable}.
The third equality follows from the inductive assumption, namely \eqref{eq:GGG-prime-full}.
The last equality follows from Lemma~\ref{lem:GGG}.

The equality for $\Prob_{\GGG'_n}$, $n\in\{2,\dots,N\}$, follows from
\begin{multline*}
  \Prob_{\GGG'_n}
  \left\{
    P_n \le \epsilon_n,
    P'_{n-1} \le \epsilon'_{n-1},
    P_{n-1} \le \epsilon_{n-1},
    \dots,
    P'_1 \le \epsilon'_1,
    P_1 \le \epsilon_1
  \right\}\\
  =
  \Expect_{\GGG'_n}
  \left(
    \Expect_{\GGG_{n-1}}
    \left(
      1_{P_n \le \epsilon_n}
      1_{P'_{n-1} \le \epsilon'_{n-1}}
      1_{P_{n-1} \le \epsilon_{n-1}}
      \dots
      1_{P'_1 \le \epsilon'_1}
      1_{P_1 \le \epsilon_1}
    \right)
  \right)\\
  =
  \Expect_{\GGG'_n}
  \left(
    1_{P_n \le \epsilon_n}
    \Expect_{\GGG_{n-1}}
    \left(
      1_{P'_{n-1} \le \epsilon'_{n-1}}
      1_{P_{n-1} \le \epsilon_{n-1}}
      \dots
      1_{P'_1 \le \epsilon'_1}
      1_{P_1 \le \epsilon_1}
    \right)
  \right)\\
  =
  \Expect_{\GGG'_n}
  \left(
    1_{P_n \le \epsilon_n}
    \epsilon'_{n-1}
    \epsilon_{n-1}
    \dots
    \epsilon'_1
    \epsilon_1
  \right)
  =
  \epsilon_n
  \epsilon'_{n-1}
  \epsilon_{n-1}
  \dots
  \epsilon'_1
  \epsilon_1.
\end{multline*}
Now the second equality follows from Lemma~\ref{lem:measurable}.
The third equality follows from the inductive assumption,
namely \eqref{eq:GGG-full} with $n-1$ in place of $n$.
The last equality follows from Lemma~\ref{lem:GGG-prime}.

Plugging $n:=N$ into \eqref{eq:GGG-full},
we obtain
\begin{equation*}
  \Prob
  \left\{
    P_1 \le \epsilon_1,
    P'_1 \le \epsilon'_1,
    \dots,
    P_N \le \epsilon_N,
    P'_N \le \epsilon'_N
  \right\}
  =
  \epsilon_1
  \epsilon'_1
  \dots
  \epsilon_N
  \epsilon'_N.
\end{equation*}
This implies the uniform distribution of $(P_1,P'_1,\dots,P_N,P'_N)$ in $[0,1]^{2N}$
(see, e.g., \cite[Lemma~2.2.3]{Shiryaev:2016}).

\section{Proof of Corollary~\protect\ref{cor:product}}
\label{app:proof-corollary}

Let the simple label-conditional conformal martingale be
\[
  S_n
  =
  F(P_1,\dots,P_n),
  \quad
  n=0,1,\dots,
\]
and the label conformal martingale be
\[
  S'_n
  =
  F'(P'_1,\dots,P'_n),
  \quad
  n=0,1,\dots,
\]
where $F$ and $F'$ are betting martingales and $P_1,P'_1,P_2,P'_2,\dots$
is a stream of p-values as in Theorem~\ref{thm:interleaved}.
Let us check that $S_n S'_n$, $n=0,1,\dots$,
is a martingale w.r.\ to the filtration generated by the p-values:
for any $n\in\{1,2,\dots\}$,
\begin{multline*}
  \Expect_{P_1,P'_1,\dots,P_{n-1},P'_{n-1}}
  (S_n S'_n)\\
  =
  \Expect_{P_1,P'_1,\dots,P_{n-1},P'_{n-1}}
  \left(
    \Expect_{P_1,P'_1,\dots,P_{n-1},P'_{n-1},P_n}
    (S_n S'_n)
  \right)\\
  =
  \Expect_{P_1,P'_1,\dots,P_{n-1},P'_{n-1}}
  \left(
    S_n
    \Expect_{P_1,P'_1,\dots,P_{n-1},P'_{n-1},P_n}
    (S'_n)
  \right)\\
  =
  \Expect_{P_1,P'_1,\dots,P_{n-1},P'_{n-1}}
  \left(
    S_n
    S'_{n-1}
  \right)\\
  =
  \Expect_{P_1,P'_1,\dots,P_{n-1},P'_{n-1}}
  \left(
    S_n
  \right)
  S'_{n-1}
  =
  S_{n-1}
  S'_{n-1},
\end{multline*}
where each lower index for $\Expect$ signifies the conditioning $\sigma$-algebra
(namely, the conditioning $\sigma$-algebra is generated by the listed random variables).
The third and last equalities follow from \eqref{eq:betting-martingale}.
\end{document}